\documentclass[letterpaper, 10 pt, conference]{ieeetran}  

\IEEEoverridecommandlockouts                              

\usepackage{subcaption}
\usepackage{enumitem}
\usepackage{graphicx} 
\usepackage{comment}
\usepackage{overpic}
\usepackage{cite}
\usepackage{algorithm}
\usepackage{algorithmic}
\usepackage{amsthm}
\usepackage{arydshln}
\newtheorem{theorem}{Theorem}

\usepackage{amsmath} 
\usepackage{amssymb}
\usepackage{bm}
\usepackage{balance}
\usepackage{hyperref}
\hypersetup{
    colorlinks=true,
    linkcolor=black,
    filecolor=blue,      
    urlcolor=black,
    pdftitle={Constrained Potential Surgery},
    }

\usepackage{xcolor}
\floatstyle{plaintop}
\restylefloat{table}
\usepackage[tableposition=top]{caption}
\usepackage{soul}

\newcommand{\T}{{\sf T}}

\title{\LARGE \bf
Game-Theoretic Control with Constrained Potential Surgery
}

\author{Zhiyuan Zhang$^{1}$ and Panagiotis Tsiotras$^{2}$
\thanks{This work is sponsored by ONR awards N00014-23-2308 and N00014-23-2353 and NSF award IIS-2008686}
\thanks{$^{1}$ School of Aerospace Engineering, Institute for Robotics and Intelligent Machines, Georgia Institute of Technology, Atlanta, GA 30332, USA, Email:
        {\tt\small zzhang615@gatech.edu}}%
\thanks{$^{2}$ School of Aerospace Engineering, Institute for Robotics and Intelligent Machines, Georgia Institute of Technology, Atlanta, GA 30332, USA, Email:
        {\tt\small tsiotras@gatech.edu}}%
}

\renewcommand{\baselinestretch}{0.98}

\begin{document}
\maketitle
\thispagestyle{empty}
\pagestyle{empty}

\begin{abstract}
Constrained general-sum dynamic games are a popular formulation for highly interactive multi-agent planning problems.
In recent years, Generalized Nash Equilibrium (GNE) solvers have achieved real-time performance for small dynamic games.
However, solution speed still remains a bottleneck, and controlling even a small number of agents (e.g., more than four) in a dynamic task remains elusive.
In addition, Newton solvers that focus on the first-order conditions are vulnerable to non-Nash saddle points, limiting the usefulness of the provided solution.
In this work, we propose a fast and versatile interior point solver for constrained dynamic games,
along with a computationally efficient second-order correction that increases the probability of converging to a local GNE solution in constrained dynamic games.
The performance of the proposed method is evaluated on numerical benchmarks and a physical experiment involving scaled race cars.

\end{abstract}

\section{Introduction}

\IEEEPARstart{C}{onstrained} general-sum dynamic games have emerged as a powerful formulation for modeling highly interactive multi-agent path-planning problems. 
In recent years, Newton solvers for Generalized Nash Equilibria (GNE) have advanced significantly, achieving real-time performance for small-scale games. 
However, solution speed remains a critical bottleneck, and experimental demonstrations are often restricted to four or fewer agents, typically involving slow-moving robots~\cite{ilqgame, algames,rd3g,gt_1,gt_2}.
In addition, standard Newton methods that solve for the first-order Karush-Kuhn-Tucker (KKT) conditions are inherently vulnerable to converging to non-Nash saddle points, severely limiting the practical usefulness of their solutions~\cite{algames, rd3g}.
To address these challenges, we propose Constrained Potential Surgery (CPS), a versatile interior-point solver designed for \textit{constrained} dynamic games. 

The primary contributions of this work are threefold:
\begin{itemize}[label={$\bullet$},leftmargin=*]

\item
\textbf{Inertia-Preserving Curvature Surgery:} 
We introduce a numerically efficient positive-definite correction to the symmetric primal block of the game Jacobian utilizing an $LDL^\top$ factorization.
This structural intervention destabilizes non-GNE saddle points, yielding a significant, simulation-verified improvement in the optimality rate of computed solutions.

\item
\textbf{High-Performance Solver Architecture:} We develop a highly optimized interior-point GNE solver that is more than 4X faster than the state-of-the-art Newton and active-set solvers, improving the computational bottleneck for online multi-agent planning.

\item
\textbf{Multi-Vehicle Physical Validation:} We validate the real-time capabilities of the proposed solver on a physical miniature autonomous racing platform, successfully controlling a fleet of six
vehicles executing complex adversarial maneuvers at receding-horizon rates.

\end{itemize}

\section{Related Work}

An extensive body of literature exists on the theory and application of dynamic games.
Rather than providing a comprehensive survey, this section focuses on relevant methods derived from solving the corresponding necessary conditions for optimality.
We restrict our attention to public-information games, where the dynamics and cost functions of all players are assumed to be common knowledge.
Furthermore, we assume that the underlying cost functions and system dynamics are $C^2$ continuous.
 For a more foundational and structured discussion on smooth dynamic games, readers are referred to \cite{smoothgame}.
 
Early approaches to continuous dynamic games often converted the underlying problem into a more tractable format.
 For instance, although the action space is inherently continuous, heuristic sampling can be utilized to approximate the continuous game as a finite matrix game \cite{bimatrix}.
 In this approach, dynamic constraints are implicitly satisfied by sampling the control variables and simulating the resulting state trajectories, while state constraints are enforced via rejection sampling.
 Although this methodology relies heavily on heuristics to construct a discrete action space that is both manageable and representative, the resulting Nash equilibrium of the sampled game can be highly effective in practice.
 Physical experiments using scaled autonomous vehicles have successfully demonstrated the real-time capabilities of this sampling-based approach~\cite{bimatrix}.
 
Another prominent strategy is Iterated Best Response (IBR) \cite{ibr}, which decouples the dynamic game into a sequence of single-agent optimal control problems.
 IBR iteratively optimizes the trajectory for one player while holding the control policies of all other players static.
 To address the slow convergence typical of standard IBR, Sensitivity-Enhanced Iterated Best Response (SE-IBR) \cite{seibr_1} accelerates the process by actively encouraging agents to explore interaction constraints.
 With careful implementation, SE-IBR has been successfully deployed in complex scenarios, such as a two-player game featuring a full-sized autonomous vehicle interacting with a simulated opponent~\cite{seibr}.
 
More recently, algorithms that directly solve for the Nash Equilibrium (NE) or Generalized Nash Equilibrium (GNE) of the continuous game have emerged~\cite{ilqgame,algames,rd3g, ilqgame_lagrange}.
 These direct solution methods generally fall into two broad categories summarized next.

\subsection{Dynamic Programming and DDP-Based Methods}

The first category of GNE solvers encompasses Differential Dynamic Programming (DDP)-based methods, such as minimax-DDP\cite{minimax_ddp} and iLQGame\cite{ilqgame}. 
These methods solve for each stage of the game sequentially and propagate the optimal policy at one stage to the cost function of the previous stage.
The formulation and solution used at each stage are the primary differentiators and determine their applicability.
For example, iLQGame approximates each state of the game as a general-sum Linear-Quadratic Game (LQGame), whose stationarity conditions can be computed analytically.
While these algorithms elegantly accommodate soft constraints through penalty formulations, they struggle to rigorously enforce hard constraints. 
Extensions that explicitly handle chance constraints or hard collision constraints exist, but they typically incur severe computational overhead, rendering them inadequate for real-time control \cite{ilqgame_lagrange}.

\subsection{Newton-Based GNE Solvers}

The second category comprises Newton-based solvers, which formulate the game as a root-finding problem over the KKT necessary conditions and handle constraints rigorously via Lagrange multipliers. 
Algorithms in this category, such as ALGAMES\cite{algames} and RD3G\cite{rd3g}, are significantly faster than their DDP-based counterparts. 
ALGAMES introduced the first multiple-shooting Newton method for dynamic games, while RD3G leverages an active-set method to further improve real-time performance. 
RD3G has demonstrated sufficient computational efficiency to support physical experiments involving two scaled autonomous cars racing at high speeds~\cite{rd3g}.

\subsection{Limitations and Structural Interventions}

Despite these advancements, two major drawbacks persist within the current state-of-the-art. 
First, because these methods rely exclusively on first-order necessary conditions, they do not verify the second-order sufficiency conditions via the game Hessian. 
For example, in iLQGames, the solution at each stage relies on the closed-form stationary point of a local quadratic game approximation, while ALGAMES solves strictly for a KKT point. 
Consequently, these solvers cannot natively distinguish between a minimizer, a maximizer, and a saddle point.
This limitation is explicitly acknowledged in the existing literature~\cite{algames,rd3g, dldu_2}.

Second, despite substantial improvements in solution times, computational efficiency remains a fundamental challenge. 
Inherently, the decision space scales exponentially with the number of agents and the length of the receding horizon. 
Without utilizing a parameterization to reduce dimensionality or mean-field methods to reformulate the game, finding a Generalized Nash Equilibrium for games involving more than a few agents remains computationally prohibitive for online, receding-horizon execution.

For narrower classes of games, specialized analytical and numerical methods exist to address these shortcomings. 
To circumvent the computational bottleneck, \textit{potential game} formulations can convert the multi-agent dynamic game into a single-objective optimization problem, yielding massive computational speed improvements \cite{potential}. 
Furthermore, to address the tendency of first-order Newton methods to converge to local maximizers or non-Nash saddle points in \textit{zero-sum games}, 
 algorithms such as Local Symplectic Surgery (LSS)\cite{lss} and DND/SECOND\cite{dnd} introduce correction terms to the Newton descent dynamics. 
These corrections provably destabilize saddle-point attractors, ensuring convergence exclusively to local Nash Equilibria. 
While the specific structural properties these algorithms rely on do not universally exist in general-sum, multi-player games, their underlying methodology provides the theoretical inspiration for the Hessian correction mechanisms proposed in this work.

\section{Problem Formulation}

Denote $x^i_k, u^i_k$ as the state and control vectors of player $i$ at stage $k$, respectively. 
For notational simplicity, we use $x^i$ to denote the concatenated state of player $i$ across all stages, $x_k$ to denote the concatenated state of all players at stage $k$, and $x$ to denote the concatenated state of all players across all stages. 
We also adopt the common notation that $x^{-i}$ denotes the concatenated state of all players at all stages except agent $i$.
The same notation abbreviation applies to other variables.
We are interested in a general-sum, multi-stage, constrained dynamic game of the form
\begin{subequations}\label{eq:game}
\begin{align}
\min_{x^i,\,u^i} \quad
& J^i(x, u^i), \label{eq:game_obj}\\
\textrm{s.t.} \quad 
& x^i_{k+1} = f(x^i_k, u^i_k), 
\quad k = 0, \dots, T-1, \label{eq:dyn_con}\\
& h(x, u) \le 0,  \label{eq:ineq_con}\\
\textrm{where  } & J^i(x, u^i)=
\sum_{k=0}^{T} J^i_k(x_k, u^i_k).
\end{align}
\end{subequations}
The feasible control set for player $i$ is defined as $\mathcal{U}^i(u^{-i}) := \{u^i | h(x, u^i, u^{-i}) \le 0, x^i_{k+1} = f(x^i_k, u^i_k) \}$. 
Let $\mathcal{N}^i(u^i)$ denote a neighborhood of $u^i$.
We are interested in finding the local Generalized Nash Equilibrium (GNE), that is, a set of controls $u^* = (u^{1*}, u^{2*}, \dots, u^{N*})$ such that
\begin{equation}
\begin{split}
\exists \mathcal{N}^i(u^{i*}), 
\quad J^i(x, u^{i*}, u^{-i*}) \le J^i(x, u^i, u^{-i*}),  \\
\quad \forall u^i \in \mathcal{U}^i(u^{-i*})\cap\mathcal{N}^i(u^{i*}).
\end{split}
\end{equation}

\section{Methodology}

\subsection{Necessary Conditions}

The Lagrangian of the dynamic game for player $i$ is\footnote{In this work, we solve for the variational GNE, where all players share the same dual multiplier ($\mu$) for the coupled inequality constraints. 
Although variational GNEs constitute a strict subset of all possible GNEs, this formulation is a standard convention in dynamic game solver literature~\cite{algames}
because it resolves the non-uniqueness of multipliers and ensures computational tractability~\cite{gnep}.}
\begin{align}\label{eq:lagrangian}
\mathcal{L}^i (x,u,\lambda^i, \mu^i) 
&= \sum_{k=0}^{T} J^i_k(x_k, u^i_k)
+ \mu^\top h(x, u)\nonumber \\ 
&\quad + \sum_{k=0}^{T-1} (\lambda^{i}_k)^\top \big[f(x_k^i, u_k^i) - x_{k+1}^i\big].
\end{align}

The first-order KKT necessary conditions for GNE are\cite{gnep}:
\begin{subequations}\label{eq:kkt}
\begin{align}
\nabla_{x^i}\mathcal{L}^i &= 0, \quad \nabla_{u^i}\mathcal{L}^i = 0, \quad \forall i \in [1..N],\label{eq:kkt_stationarity}\\
\mathcal{F}^i(x^i,u^i) &= 0, \quad \forall i \in [1..N],\label{eq:kkt_dynamics}\\
h(x, u) &\le 0, \label{eq:kkt_ineq}\\
\mu &\ge 0, \quad \mu^\top h(x,u) = 0,\label{eq:kkt_comp}
\end{align}
\end{subequations}
where (\ref{eq:kkt_stationarity}) are the stationarity conditions, (\ref{eq:kkt_dynamics}), (\ref{eq:kkt_ineq}) are the primal feasibility conditions, and (\ref{eq:kkt_comp}) are the dual feasibility and complementary slackness conditions.
In (\ref{eq:kkt_dynamics}), $\mathcal{F}^i$ stands for the stacked dynamics constraint for player $i$, as defined in (\ref{eq:dyn_con}).

\subsection{Performance Considerations}

%
%
Many standard algorithms such as~\cite{algames,ilqgame_barrier1,ilqgame_barrier2,rd3g}  often struggle to find a solution if they are not provided with a feasible initial guess, 
a task that is notoriously difficult to satisfy for tightly constrained multi-agent problems.
Inspired by modern interior point methods such as IPOPT~\cite{ipopt}, we introduce a slack variable $s$,
to achieve robustness against highly infeasible initializations,
and solve, instead, the modified set of conditions as shown below:
\begin{subequations}\label{eq:kkt_slack}
\begin{align}
\nabla_{x^i}\mathcal{L}^i &= 0, \quad \nabla_{u^i}\mathcal{L}^i = 0,\quad \forall i \in [1..N], \label{eq:kkt_stationarity2}\\
\mathcal{F}^i(x^i,u^i) &= 0,\quad \forall i \in [1..N], \label{eq:kkt_dynamics2}\\
h(x, u) + s&= 0, \label{eq:kkt_ineq2}\\
s &\ge 0,\\
\mu &\ge 0, \quad \mu^\top h(x,u) = \tau.\label{eq:kkt_comp2}
\end{align}
\end{subequations}
By introducing the slack variable $s$, the solver can satisfy the non-negativity conditions $s \ge 0$ and $\mu \ge 0$ at every iteration. Consequently, any intermediate constraint violations are localized entirely within the equality residuals (\ref{eq:kkt_dynamics2}) and (\ref{eq:kkt_ineq2}). 
Furthermore, we use a perturbed complementary slackness condition in (\ref{eq:kkt_comp2}). 
When solving for the strict complementarity condition $\mu^\top h = 0$, the Newton descent step involves solving for the condition $\mu^\top \Delta s + s^\top \Delta \mu = \mu^\top s$.
As the solver approaches the boundary of the feasible region, either $s \to 0$  or $\mu \to 0$.
If either approaches zero, the corresponding row in the KKT matrix approaches zero, making the matrix ill-conditioned or singular.
Using the perturbed complementary slackness avoids the numerical sensitivity associated with exact complementary slackness, thereby improving algorithmic stability~\cite{nocedal2006numerical}. 
The homotopy parameter $\tau > 0$ is gradually shrunk toward zero as the solver progresses.

Define the game residual as $\mathcal{R} = %
[\nabla_{x}\mathcal{L} ,%
\nabla_{u}\mathcal{L},%
\mathcal{F},%
h + s,%
\mu^\top h - \tau]$, where the stationarity residual ($\nabla_x \mathcal{L}, \nabla_u \mathcal{L}$),  and the dynamics constraint residual $\mathcal{F}$ are formed by stacking the respective residuals for all players.
Denote the search direction as $\Delta r = [\Delta x, \Delta u, \Delta\lambda, \Delta\mu]$\footnote{In (\ref{eq:newton}), $s$ can be eliminated with (\ref{eq:kkt_ineq2}) and (\ref{eq:kkt_comp2}), reducing the size of the linear program.}.
The Newton method finds the descent direction by solving
\begin{equation}\label{eq:newton}
\nabla\mathcal{R}^{(t)} \,\Delta r + \mathcal{R}^{(t)} = 0,
\end{equation}
where $\mathcal{R}^{(t)}$ is the residual at iteration $t$. 

\subsection{Inertia Adjustment}

Existing Newton solvers for dynamic games \cite{algames, rd3g} often take the solution to (\ref{eq:kkt}) as the game solution.
However, 
solving (\ref{eq:kkt}) yields only a stationary point,
and may not satisfy the sufficient conditions to guarantee a minimizer for all players.
When the game Jacobian $\nabla \mathcal{R}$ contains negative eigenvalues, $\Delta r$ can be an ascent direction along the corresponding eigenvectors of $\nabla \mathcal{R}$, attracting the solver to local maxima or saddle points.
To improve convergence to a local minimum for all players (a Generalized Nash Equilibrium), we must encourage the solver to take a strict descent step in the primal space and destabilize saddle points in the solver dynamics.
We propose to achieve this by applying a numerically efficient operation to the game Jacobian.

Let the values of the primal and dual variables at the current iteration be denoted by 
as $\bar{p} = [\bar{x},\bar{u}]$ and $\bar{d} = [\bar{\lambda}, \bar{\mu}]$, respectively.
Assume that $\bar{p}$ is feasible.  
Then, the Newton's problem in (\ref{eq:newton}) takes the form
\begin{equation}\label{eq:expanded_newton}
\begin{bmatrix} J_p(\bar{p},\bar{d}) & J_d^\top (\bar{p},\bar{d}) \\ 
J_d (\bar{p},\bar{d}) & 0 
\end{bmatrix} 
\begin{bmatrix} 
\Delta p \\ \Delta d 
\end{bmatrix} = 
\begin{bmatrix} 
-\nabla_p \mathcal{L}(\bar{p},\bar{d}) \\ 0 
\end{bmatrix}.
\end{equation}
We decompose $J_p$ into its symmetric and skew-symmetric components.
Denote the symmetric component of $J_p$ as $S = (J_p + J_p^\top)/2$, and denote its skew-symmetric component as $A = (J_p - J_p^\top)/2$.
Henceforth, we will drop the arguments $(\bar{p},\bar{d})$ to keep the notation simple, but the reader should be reminded that all derivatives are evaluated at the current iteration step.

We perform an $LDL^\T$ factorization on $S$ such that $S = L D L^\T$, where $L$ is a unit lower triangular matrix and $D$ is a diagonal matrix.
Let $|D|$ denote the matrix formed by taking the absolute value of each element in $D$, such that $|D|_{ii} = |D_{ii}|$.
We construct the modified (1,1) primal block of \eqref{eq:expanded_newton}  as $J'_p = S' + A$, where $S' = L |D| L^\T$ and
$J'_p$ replaces $J_p$ in \eqref{eq:expanded_newton}.
 
While one could theoretically compute the full eigendecomposition $S = V \Lambda V^\T$ and construct $V \Lambda V^\T$ to perfectly preserve the principal axes of the quadratic model, doing so would require expensive iterative algorithms.
The eigendecomposition requires approximately $9n^3$ floating-point operations (flops), while the $LDL^\T$ decomposition requires only $n^3/3$ flops\cite{matrix_computation}.
Furthermore, $LDL^\T$ factorization executes in deterministic time and provides a highly efficient, structure-preserving mechanism for inertia-controlling regularization, making it vastly more practical for online solver frameworks.

By Sylvester's Law of Inertia~\cite{matrix_computation}, the number of negative elements in $D$ is exactly equal to the number of negative eigenvalues in $S$.
Therefore, modifying $D$ to $|D|$ guarantees that $S'$ is strictly positive-definite and that $J'_p$ has eigenvalues with positive real part, actively pushing the descent dynamics away from saddle points and maxima.
Because the $LDL^\T$ factorization avoids iterative eigenvalue computation, it provides the required robust descent direction at a fraction of the computational cost, making it highly desirable for real-time optimal control implementations.

\subsection{Theoretical Analysis}

Let $Z$ be a matrix
whose columns form a basis for the null space of the constraint matrix $J_d$.
From (\ref{eq:expanded_newton}), it then follows that $J_d \Delta p = 0$.
Hence, any valid descent step can be written as $ \Delta p = Z \Delta p_N$ for some $\Delta p_N$.
Substituting the last expression back to \eqref{eq:expanded_newton} yields
\begin{equation}
J_p Z \Delta p_N + J_d^\top \Delta d = -\nabla_p \mathcal{L}.
\end{equation}
Pre-multiplying the previous equation by $Z^\top$ yields
\begin{equation}  \label{eq:reduced}
Z^\top J_p Z \Delta p_N = -Z^\top \nabla_p \mathcal{L}.
\end{equation}
%
%
We may therefore consider the reduced problem in \eqref{eq:reduced}
where $J_p$ and $\nabla_p \mathcal{L}$ are projected onto the constraint null space via the basis matrix $Z$.
Henceforth, we denote the projection onto the constraint null space with a subscript $N$, and rewrite \eqref{eq:reduced}
as
\begin{equation}\label{eq:gn_def}
J_{p,N}\Delta p_N = -\nabla_{p,N} \mathcal{L} = - g_N,
\end{equation}
where $J_{p,N} = Z^\top J_p Z$ and  the reduced pseudo-gradient is $g_N = Z^\top \nabla_p \mathcal{L}$.
Our proposed method finds the descent step with the modified game Jacobian $J_{p,N}'$
\begin{equation}\label{eq:update}
J_{p,N}'\Delta p_N = - g_N.
\end{equation}

\begin{theorem}[Reduced Primal Descent for Game Dynamics] \label{thm:descent}
The reduced Newton step $\Delta p_N$ in \eqref{eq:update}
satisfying 
$J_{p,N}' \Delta p_N = -g_N$
is a strict descent direction with respect to the reduced pseudo-gradient $g_N$, provided that $g_N \neq 0$ and $S = (J_p + J_p^\top)/2$ is non-singular.
\end{theorem}

\begin{proof}
Since $S$ is non-singular, $D$ is also non-singular, and hence all diagonal entries of $|D|$ are strictly positive. 
Consequently, the matrix $S^{\prime} = L|D|L^{\top}$ is strictly positive-definite.
To evaluate the descent property of the reduced step $\Delta p_N$, we evaluate its inner product with the reduced pseudo-gradient $g_N= -J^{\prime}_{p,N} \Delta p_N$ 
from \eqref{eq:gn_def} to obtain
\begin{equation} \label{eq:descent_inner}
g_N^{\top} \Delta p_N = -\Delta p_N^{\top} J^{\prime}_{p,N} \Delta p_N.
\end{equation}
Recall that the game Jacobian $J_p$ is decomposed into the symmetric $S$ 
and skew-symmetric $A$.
The projected Jacobian $ J_{p,N} = Z^\top J_p Z$ is similarly decomposed to its symmetric $S_N = Z^\top S Z$ and the skew-symmetric $A_N = Z^\top A Z$ parts.
The projected and modified Jacobian is $J_{p,N}^{\prime} = S^{\prime}_N + A_N$. 
Substituting this into \eqref{eq:descent_inner} yields
\begin{equation}\label{eq:12}
\begin{split}
g_N^{\top} \Delta p_N &= -\Delta p_N^{\top} S^{\prime}_N \Delta p_N - \Delta p_N^{\top} A_N \Delta p_N \\
& -\Delta p_N^{\top} S^{\prime}_N \Delta p_N < 0,
\end{split}
\end{equation}
since ${S}^{\prime}_N = Z^\top S' Z$ is positive definite.
Finally, inequality~\eqref{eq:12} guarantees strict descent.
\end{proof}

The update step \eqref{eq:update} leads to a sequence of candidate solutions $p_N^{(k+1)} = p_N^{(k)} - \alpha (J_{p,N}')^{-1}g_N(p_N^{(k)})$, where $\alpha$ is the step size.
By treating $\alpha$ as a discrete time increment $\Delta t$, and taking the limit as $\Delta t \to 0$, the sequence converges to the continuous solver dynamics
\begin{equation}\label{eq:solver_dyn}
\dot{p}_N = -(J'_{p,N})^{-1}g_N(p_N).
\end{equation}
%
%
To analyze local stability, we linearize this system around an equilibrium point $p_N^*$.
Applying the product rule to evaluate the Jacobian of the dynamics at $p_N^*$, the term involving the derivative of $(J'_{p,N})^{-1}$ vanishes because the first-order conditions are satisfied ($g_N(p_N^*) = 0$).
Since the Jacobian of the pseudo-gradient is $\nabla_{p_N} g_N = J_{p,N}$, the Jacobian of the solver dynamics at the equilibrium reduces to $M = -(J'_{p,N})^{-1} J_{p,N}$.

\begin{theorem}[Saddle Point Avoidance] \label{thm:saddle_evasion}
Let $p_N^*$ be a fixed point of the solver dynamics \eqref{eq:solver_dyn} where the first-order necessary conditions are satisfied ($g_N(p_N^*) = 0$), and the reduced game Jacobian $J_{p,N}$ evaluated at $p_N^*$ has at least one eigenvalue with a negative real part.
Assume that $M = -(J'_{p,N})^{-1} J_{p,N}$ 
has at least one unit-length eigenvector $v \in \mathbb{C}^m$ ($\|v\|_2 = 1$) that satisfies
\begin{equation}
-v^* S_N v > \frac{\|A_{N}\|_2^2}{\lambda_{\min}(S'_{N})}.
\end{equation}
Then, $p^*$ is unstable under the solver dynamics~\eqref{eq:solver_dyn}.
\end{theorem}

\begin{proof}
We will show that, under the assumptions of the theorem, $M$ has at least one eigenvalue with a strictly positive real part.
To this end, let $\lambda \in \mathbb{C}$ be the eigenvalue corresponding to the eigenvector $v$ satisfying the assumptions in the theorem. 

The eigenvalue equation $Mv = \lambda v$ expands to
$$-(S_N + A_N)v = \lambda(S'_N + A_N)v.$$

Pre-multiplying by the conjugate transpose $v^*$ yields
\begin{equation}\label{eq:eigenval}
-v^* S_N v - v^* A_N v = \lambda(v^* S'_N v + v^* A_N v).
\end{equation}
Let $a = v^* S_N v$, $b = v^* S'_N v$, and $v^* A_{N} v = ic$ where $a,c,b \in \mathbb{R}$.
Substituting into (\ref{eq:eigenval}) yields
$-(a + ic) = \lambda(b + ic).$
Solving for the real part of the eigenvalue gives
$$\text{Re}(\lambda) = - \frac{(ab+c^2)}{b^2+c^2}.$$
%
Since $c^2 \le \|A_N\|_2^2$ and $b \ge \lambda_{\min}(S'_N) > 0$ it follows that
$$\frac{c^2}{b} \le \frac{\|A_N\|_2^2}{\lambda_{\min}(S'_N)} < -v^* S_N v = -a.$$
Hence, $ab +c^2 <0$ and thus $\text{Re}(\lambda) > 0$.
It follows that $p^*$ is an unstable equilibrium of the solver dynamics~\eqref{eq:solver_dyn}.
\end{proof}

We note that while the proposed CPS method prevents the solver from converging to a class of saddle points, it cannot eliminate all saddle points.
If the adversarial interactions dominate in a game, the corresponding skew-symmetric component $A_N$ can stabilize a saddle point despite the positive-definite correction applied to the symmetric curvature. 
This is a known result in classical mechanics, where rotational forces corresponding to skew-symmetric components in the dynamics can stabilize an otherwise unstable equilibrium \cite{bloch1994dissipation} and has been formally explored in the context of rotational game dynamics and spurious attractors \cite{balduzzi2018mechanics, mescheder2017numerics}.

\subsection{Filter Line Search}

Solving the modified KKT system yields a search direction. 
To select an appropriate step size, we utilize a filter line search method inspired by modern interior-point solvers~\cite{ipopt, wachter2006implementation, fletcher2002nonlinear}.
First, to maintain the strict positivity of the slack variables $s > 0$ and the dual multipliers $\mu > 0$, we employ a fraction-to-boundary rule~\cite{nocedal2006numerical, wright1997primal}.
The maximum allowable step size $\alpha_{\max}$ is computed as
\begin{equation}\label{eq:ftb}
    \begin{split}
       \alpha_{\max} &= \max_{\alpha \in (0,1]} \alpha, \\
       &s^{(t)} + \alpha\,\Delta s > \epsilon s^{(t)} ,\\
       &\mu^{(t)} + \alpha\,\Delta\mu > \epsilon \mu^{(t)} ,\\
    \end{split}
\end{equation}
where $\epsilon$ is a small positive value to ensure that  $s$ and $\mu$ after the step are still positive.

The fundamental concept of the line search is to view step acceptance as a dual-objective optimization problem that balances feasibility and optimality. 
We evaluate trial points using two distinct $L_1$-norm metrics. 
The primal residual (measuring the infeasibility of \eqref{eq:kkt_dynamics2} and \eqref{eq:kkt_ineq2}) 
is defined as
\begin{equation}
    \theta = \| \mathcal{F} \|_1 + \| h(x, u) + s \|_1.
\end{equation}
The dual residual (measuring optimality and complementarity) is defined as
\begin{equation}
    \phi = \| \nabla_{x} \mathcal{L} \|_1 + \| \nabla_{u} \mathcal{L} \|_1 + \| \mu^\top s - \tau \|_1,
\end{equation}
where $\tau$ corresponds to the current perturbed complementary slackness target.

At each line search iteration, a trial point is checked against two criteria:
whether it improves the optimality of the solution and whether it improves the primal feasibility of the solution.
If either criterion is met, then the trial point is admitted.
Formally, the criterion is
\begin{equation}
\phi_{\text{trial}} < \phi - \gamma \theta \quad \textrm{or}
\quad \theta_{\text{trial}} < (1 - \gamma) \theta,
\end{equation}
where $\gamma \in (0, 1)$ is a tuning parameter. 

Furthermore, to prevent the algorithm from cycling or regressing, we maintain a record of $(\theta, \phi)$ pairs from previously accepted steps. A trial point is rejected if it is dominated by any entry in the filter (i.e., if it fails to improve upon both the feasibility and optimality of a historical point). If a trial step is rejected either by the current iterate or the filter, the step size $\alpha$ is reduced by a backtracking factor $\beta \in (0, 1)$, and the process is repeated. 

If a trial point is accepted solely by improving feasibility at the expense of optimality, the current coordinate pair 
$(\max(c_\theta, \theta), \phi)$ is permanently added to the filter to restrict future iterations from regressing back to this state, where $c_\theta$ is a small positive value.

\subsection{Optimality Checking}

Newton methods solving for the first-order conditions are inherently vulnerable to saddle points and maximizers.
To validate a candidate solution as a true Generalized Nash Equilibrium, we evaluate the second-order sufficient conditions for each player independently.
We define the player-specific KKT vector as $\mathcal{R}^i = [\nabla_{x^i} \mathcal{L}^i, \nabla_{u^i} \mathcal{L}^i, \mathcal{F}^i, \mathcal{H}^i_\textrm{active}]^\top$, where $\mathcal{H}^i_\textrm{active}$ concatenates all active inequality constraints involving $x^i$ and $u^i$.
The KKT matrix for player $i$ is formed by taking the Jacobian of $\mathcal{R}^i$ with respect to the primal and active dual variables, namely, $K^i = \nabla_{(x^i, u^i, \lambda^i, \mu_{i, \textrm{active}})} \mathcal{R}^i$, where $\mu_{i, \textrm{active}}$ denotes the multipliers corresponding to the active constraints.
A sufficient condition for a strict local minimum is that the inertia of this player KKT matrix is $((n+m)T, n_\textrm{active}, 0)$ \cite{gould}—meaning the number of positive eigenvalues equals the number of primal decision variables, and the number of negative eigenvalues equals the number of active constraints. 
Because congruence transformations preserve inertia, this condition is verified efficiently using an $LDL^\top$ factorization.

\subsection{Algorithm}
\begin{algorithm}
\caption{Constrained Potential Surgery (CPS)}
\label{alg:cps}
\begin{algorithmic}
\REQUIRE Initial $x, u$, Initial Homotopy $\tau^{(0)} >0$, Shrink rate $\kappa \in (0,1)$, Tolerance $\epsilon$
\ENSURE Solution $x^*, u^*$, Optimality Flag $\sigma$
\STATE \textbf{Initialize:} $s \leftarrow \max( c_s, -h(x,u) ), \mu \leftarrow \tau^{(0)}/s, \lambda \leftarrow \mathbf{0}$ 
\STATE Filter $\mathcal{F} \leftarrow \{(\theta^{(0)}, -\infty)\}$ 
\REPEAT
    \STATE $\sigma \leftarrow $ True \COMMENT{Optimality Checking}
    \STATE Compute game residual $\mathcal{R}$, Jacobian $\nabla \mathcal{R}$
    \FOR{$i=1$ \TO $N$}
        \STATE Form player $i$'s active KKT matrix $K^i$ from $\nabla \mathcal{R}$
        \IF{$\text{inertia}(K^i) \neq (T(n+m), n_\textrm{active}, 0)$}
            \STATE $\sigma \leftarrow \text{False}$
        \ENDIF
    \ENDFOR
    \STATE Apply inertia corrections $H'_S \leftarrow L |D| L^\top$ to $\nabla \mathcal{R}$
    \STATE Solve modified Newton system for search direction $\Delta$
    \COMMENT{Filter Line Search}
    \STATE Compute $\alpha \leftarrow \alpha_{\max} \leftarrow (\ref{eq:ftb})$, initial metrics $(\theta, \phi)$
    \REPEAT
        \STATE Compute trial metrics $(\theta_{\text{trial}}, \phi_{\text{trial}})$ at $y + \alpha \Delta$
        \IF{trial step dominated by $\mathcal{F}$ \OR no sufficient decrease}
            \STATE $\alpha \leftarrow \beta \alpha$ 
        \ELSE
            \STATE $(x, u, \lambda, \mu, s) \leftarrow (x, u, \lambda, \mu, s) + \alpha \Delta$
            \IF{step accepted solely via feasibility improvement}
                \STATE Add $(\max(c_\theta, \theta), \phi)$ to filter $\mathcal{F}$
            \ENDIF
            \STATE \textbf{break} line search
        \ENDIF
    \UNTIL{step accepted \OR max backtracks reached}
    \STATE $\tau \leftarrow \kappa \tau$ \COMMENT{Shrink perturbed complementary slackness}
\UNTIL{$\|\mathcal{R}\| < \epsilon$}
\RETURN $x, u, \sigma$
\end{algorithmic}
\end{algorithm}

Algorithm~\ref{alg:cps} details the core logic of the proposed method.
The process begins with an explicit optimality check, evaluating the inertia of each player's active KKT matrix to verify local second-order sufficiency. 
The solver then applies the inertia-preserving $LDL^\top$ factorization to correct the symmetric primal block, yielding a modified Newton descent direction designed to evade non-GNE saddle points. 
To guarantee robust convergence from highly infeasible initializations, the algorithm utilizes a filter line search. 
Finally, the perturbed complementarity target $\tau$ is monotonically reduced until the KKT residuals fall below the specified convergence tolerance.

\section{Numerical Simulation}

The performance of the proposed solver is first evaluated in a highway merging problem with multiple cars.
The cars follow the kinematic bicycle model~\cite{kinematic}
where $x^i_k = [p_x, p_y, \psi, v_x, v_y]$, $u^i_k = [a_x, \delta]$, where $a_x$ denotes the longitudinal acceleration, and $\delta$ denotes the steering angle.
The cost function for player $i$ is
\begin{equation}
\begin{split}
    J^i_k(x_k, u^i) = ( C_P \sum_{j\neq i}p_x^j -p_x^i) + C_v(v_x^i - v^i_{\textrm{ref}})^2 + \\
    x^{i\top}_k Qx^i_k + u^{i\top} R u^i,
\end{split}
\end{equation}
which penalizes non-ego players' lead over ego player, deviation from a player-specific target speed, large heading, high lateral velocity, and control effort.
Figure~\ref{fig:merge} shows an illustration of the 8-car merging problem.
\begin{figure}[t!]
  \centering
  \begin{subfigure}{\linewidth}
  \centering
  \includegraphics[clip, trim=300 52 280 52,angle=-90,width=0.95\linewidth]{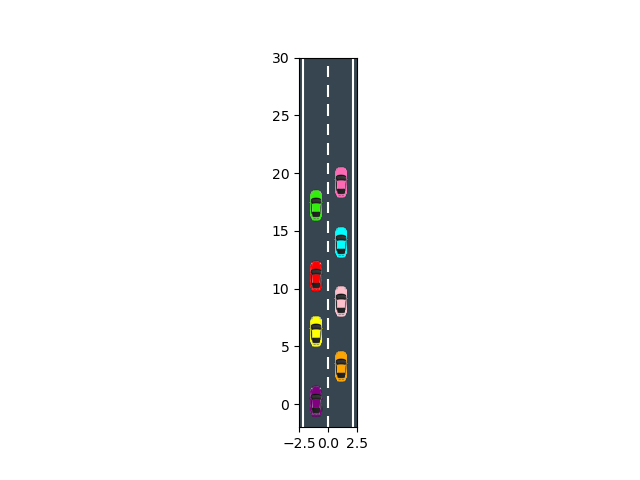}
    \caption{Initial State ($k=0$)}
  \end{subfigure}
  \begin{subfigure}{\linewidth}
  \centering
  \includegraphics[clip, trim=300 52 280 52,angle=-90,width=0.95\linewidth]{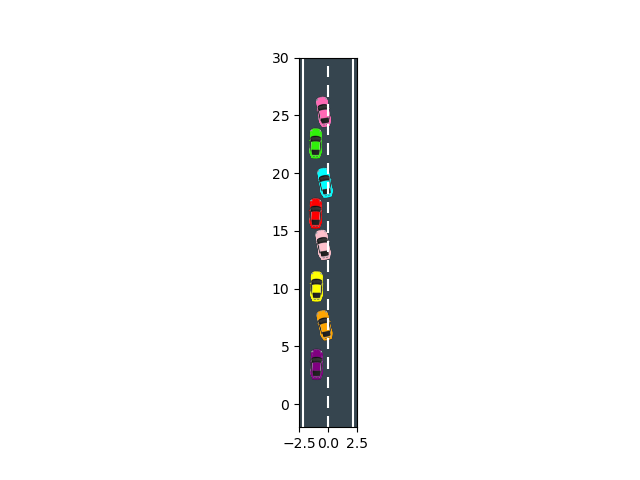}
    \caption{Intermediate State ($k=T/2$)}
  \end{subfigure}
  \begin{subfigure}{\linewidth}
  \centering
  \includegraphics[clip, trim=300 52 280 52,angle=-90,width=0.95\linewidth]{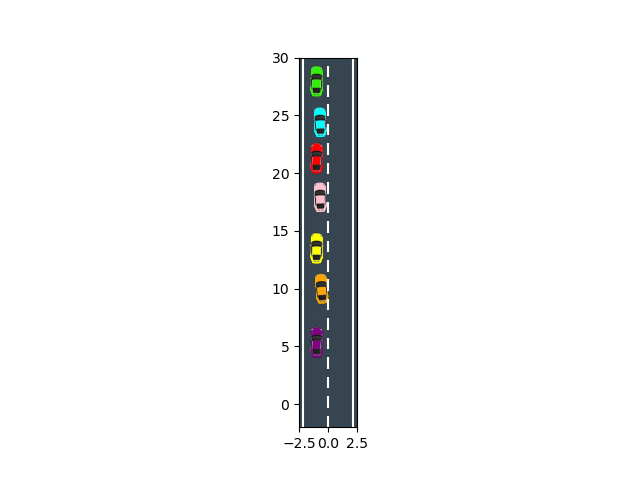}
    \caption{Final State ($k=T$)}
  \end{subfigure}
  \caption{Time-lapse of a solution to the 8-car merging game. Cars merge to the upper lane while maintaining safety distances.}
  \label{fig:merge}
\end{figure}

The initial state of the cars and the target speed for each car are randomly generated.
The horizon is set at 20 steps.
Figure~\ref{fig:runtime} compares the average runtime of the baseline RD3G algorithm, the proposed interior point method (IPM), and the proposed method with Constrained Potential Surgery (IPM+PS) for two to eight total cars over 100 scenarios.
All solvers used a stopping tolerance of $10^{-4}$ with 30 iterations.
As shown in Figure~\ref{fig:runtime}, the proposed method's runtime advantage over the baseline~\cite{rd3g}, which is reportedly faster than iLQGame~\cite{ilqgame} and ALGames~\cite{algames}, is significant.

\begin{figure}
    \centering
    \includegraphics[width=\linewidth]{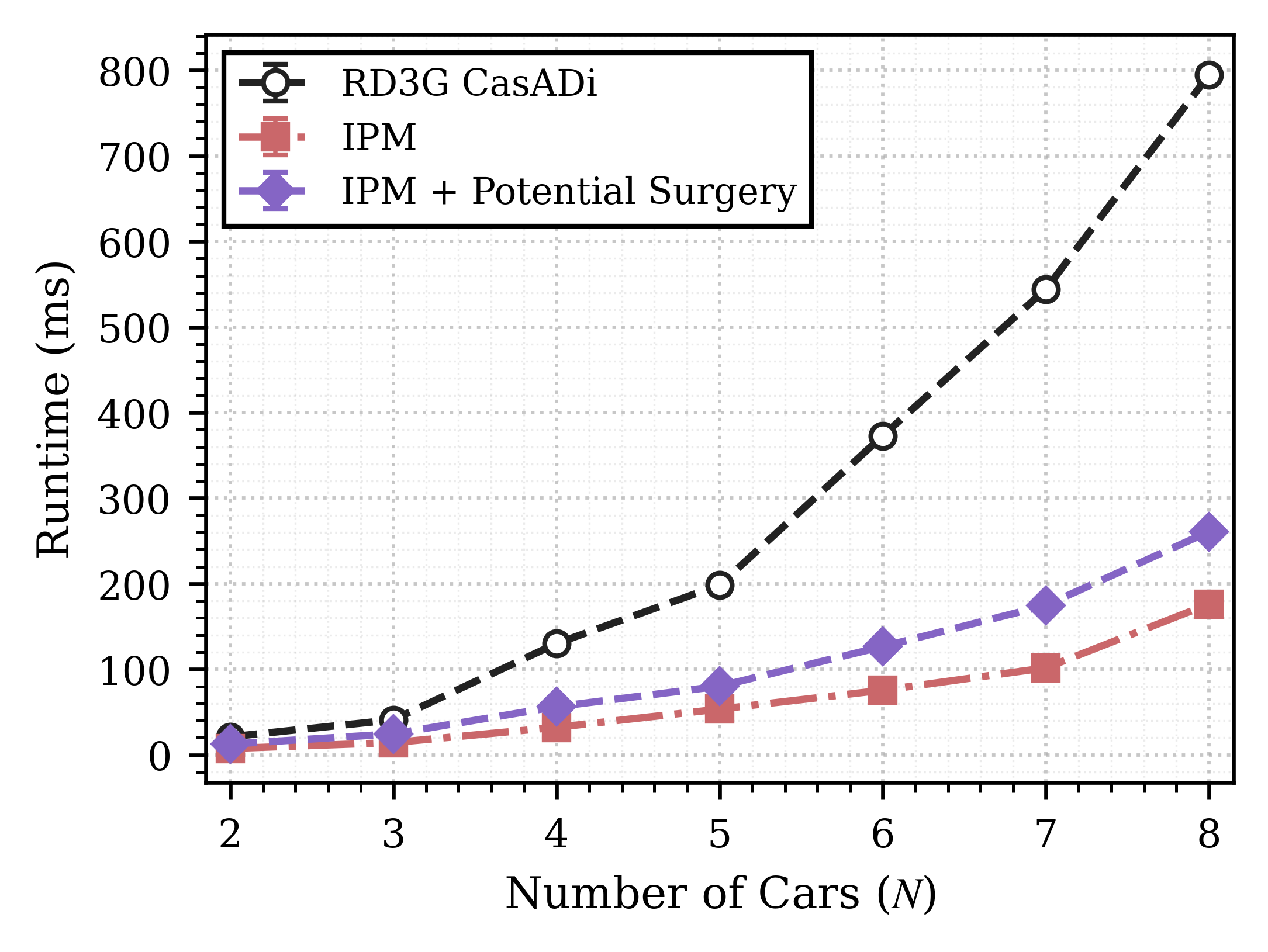}
    \caption{Average runtime comparison of baseline and proposed algorithms.}
    \label{fig:runtime}
\end{figure}

\color{black}
Figures \ref{fig:convergence} and \ref{fig:optimality} illustrate the convergence and optimality rates of the proposed algorithms.
Convergence here means that the solver reached a solution where the first-order conditions are met, with KKT residuals below the specified tolerance. 
Optimality means that the control actions for all agents are at least local minima of their respective objective functions, as verified by checking the second-order conditions for each agent.
As expected, both metrics decrease as the number of agents increases, reflecting the growing complexity of the game and the exponential proliferation of non-Nash saddle points.
Because Newton-type solvers rely on local gradient and curvature information, their convergence is inherently restricted to the local region of attraction of the initial guess~\cite{nocedal2006numerical}. 
In tightly constrained, multi-agent scenarios, providing an initial guess within this region is difficult, making guaranteed convergence difficult without random restarts.

Notably, while the proposed potential surgery approach incurs a slight reduction in the overall convergence rate, it yields a significant increase in the optimality ratio. 
This trade-off is deliberate: the surgery actively destabilizes non-Nash attractors, forcing the solver to seek true local GNEs. 
Consequently, cases that would have previously converged to poor, non-Nash saddle points are either successfully redirected to optimal solutions or gracefully fail to converge. 
Ultimately, this degradation in raw convergence is greatly outweighed by the improvement in solution quality, yielding a higher absolute number of optimal, dynamically viable trajectories, with a small runtime penalty.
\color{black}

All benchmarks were completed on a desktop computer equipped with an i7-7700K CPU @ 4.20GHz and 32GB RAM.

\begin{figure}
    \centering
    \includegraphics[width=\linewidth]{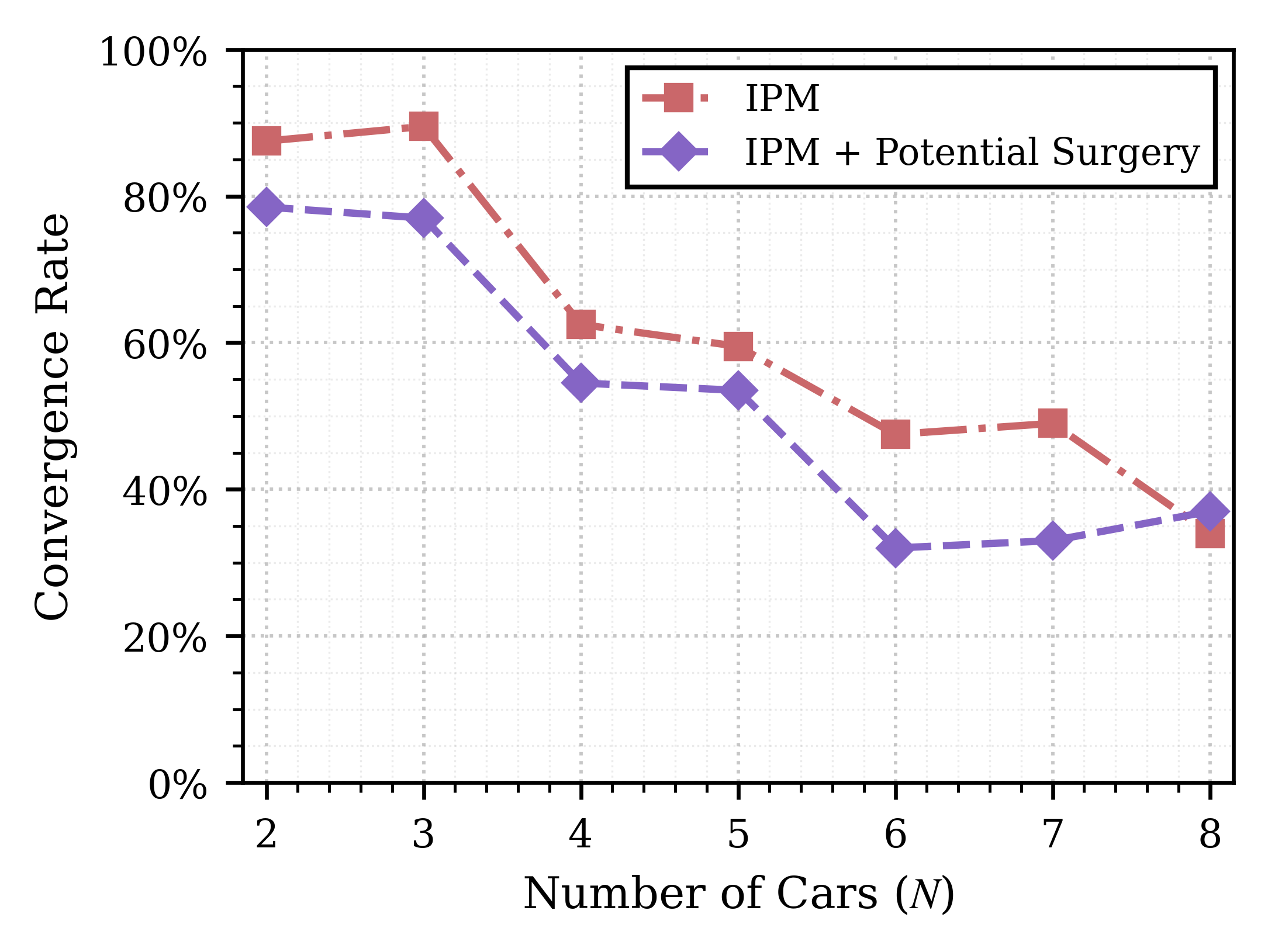}
    \caption{Convergence ratio of proposed algorithms.}
    \label{fig:convergence}
\end{figure}
\begin{figure}
    \centering
    \includegraphics[width=\linewidth]{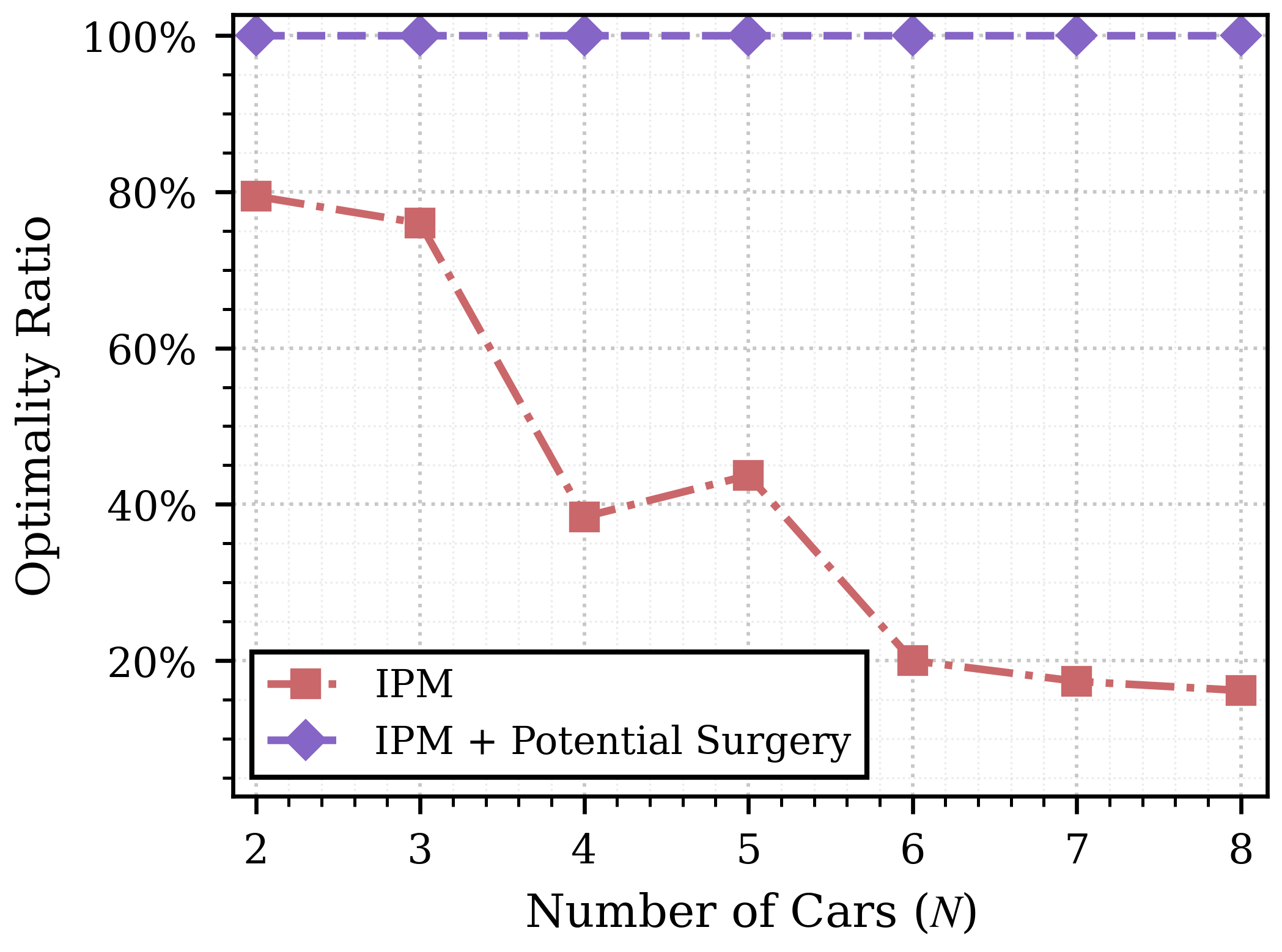}
    \caption{Optimality ratio of the converged simulations.}
    \label{fig:optimality}
\end{figure}

\section{EXPERIMENTAL EVALUATION}
\label{sec:experiment}

To validate the real-world applicability and real-time performance of the proposed solver, we conducted a physical experiment involving a multi-agent autonomous racing scenario.
The experiment utilizes BuzzRacer~\cite{buzzracer}, a miniature autonomous racing platform designed for testing adversarial decision-making.

The physical system is modeled using a kinematic bicycle model formulated within the Frenet frame~\cite{kinematic}.
The objective function for each agent is designed to encourage competitive interaction, incorporating rewards for maintaining a lead over opponents and penalties for control effort, heading deviation, lateral deviation from the reference centerline, and longitudinal deviation from the reference speed.
To induce overtaking and blocking behaviors, heterogeneous reference speeds were assigned to the vehicles.
Safety and physical limits were enforced via inequality constraints.
 Track boundaries were modeled as state constraints, and inter-agent collision avoidance was handled by approximating each vehicle's footprint as two tangent circles.
 A collision constraint is violated if any circle of one agent intersects with a circle of another agent.

The software architecture employs a hierarchical planning and control pipeline.
 At the high level, a CPS-based planner solves for the GNE trajectories.
 To mitigate the risk of convergence failure and topological entrapment, the planner concurrently runs two CPS solvers initialized with distinct trajectories: 
 a zero-control guess and a simulated, obstacle-ignorant Stanley controller~\cite{stanley} rollout. 
 While warm-starting from a prior iteration is common in single-agent trajectory optimization, it is highly prone to poor local optima in tightly constrained, multi-agent dynamic games.
 For example, a warm-started solver may continuously attempt to close the gap between adversarial vehicles rather than discovering a topologically distinct path to circumvent them.
 Concurrent initialization using varied heuristic guesses effectively circumvents this issue, ensuring the real-time execution of the safest and most optimal maneuver.
 This top-level planner achieves an average solution time of 120 ms.
 %
%
At the low level, the trajectory generated by the CPS planner is tracked by a Stanley controller operating at 100 Hz.
Figure \ref{fig:experiment_snapshots} illustrates a sequence of snapshots demonstrating a successful three-car overtaking maneuver executed during the physical trials.
 The accompanying supplementary video provides a comprehensive view of the multi-agent racing behaviors.

\begin{figure}[thpb]
  \centering
  \begin{subfigure}[b]{0.3\linewidth}
  \includegraphics[width=0.95\linewidth]{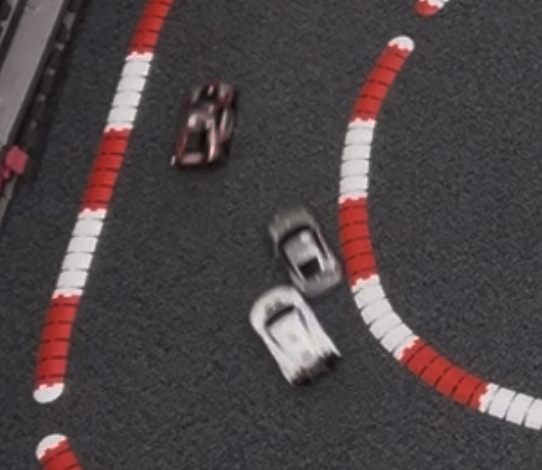}
  \caption{}
  \end{subfigure}
  \begin{subfigure}[b]{0.3\linewidth}
  \includegraphics[width=0.95\linewidth]{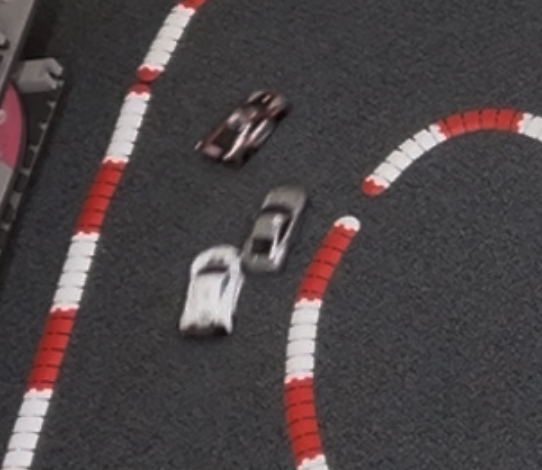}
  \caption{}
  \end{subfigure}
  \begin{subfigure}[b]{0.3\linewidth}
  \includegraphics[width=0.95\linewidth]{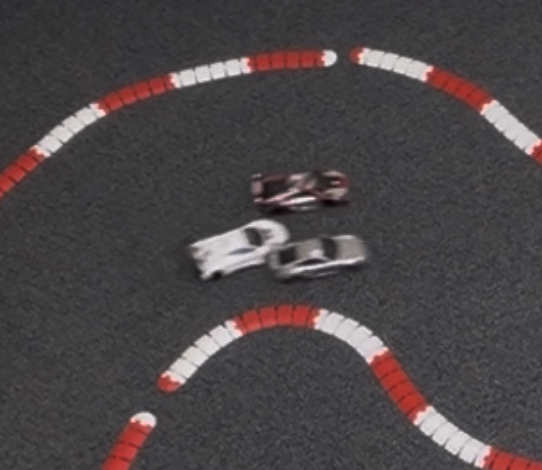}
  \caption{}
  \end{subfigure}
  \begin{subfigure}[b]{0.3\linewidth}
  \includegraphics[width=0.95\linewidth]{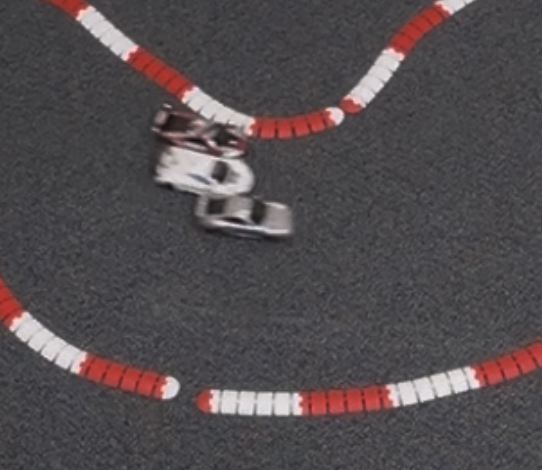}
  \caption{}
  \end{subfigure}
  \begin{subfigure}[b]{0.3\linewidth}
  \includegraphics[width=0.95\linewidth]{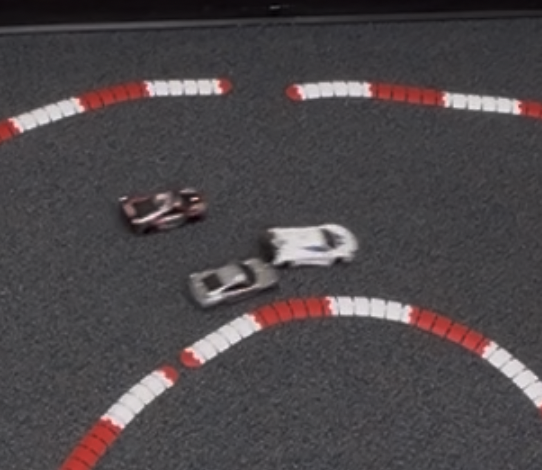}
  \caption{}
  \end{subfigure}
  \begin{subfigure}[b]{0.3\linewidth}
  \includegraphics[width=0.95\linewidth]{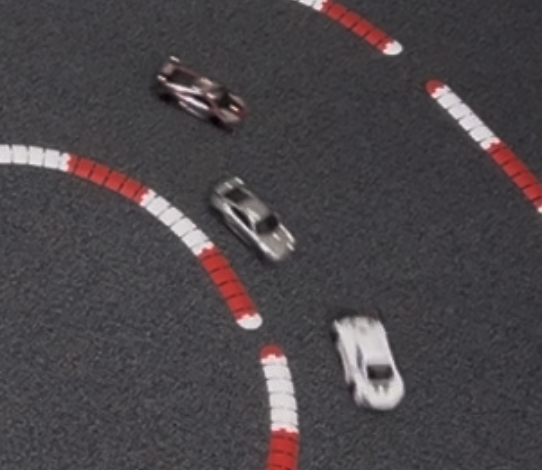}
  \caption{}
  \end{subfigure}
  \caption{Snapshots of a three-car overtake during the experiment.}
  \label{fig:experiment_snapshots}
\end{figure}

\section{CONCLUSION}
\label{sec:conclusion}

In this work, we presented Constrained Potential Surgery (CPS), a highly efficient interior-point solver tailored for constrained general-sum dynamic games.
 By explicitly correcting the symmetric part of the game Hessian via an efficient $LDL^{\top}$ factorization, the proposed method actively steers the descent direction away from non-Nash saddle points.
 This structural intervention significantly improves the optimality rate of the computed solutions with minimal computational cost.
Numerical benchmarks in a highly interactive highway merging scenario demonstrate that CPS scales effectively, maintaining a distinct runtime advantage over existing state-of-the-art Newton and active-set solvers while achieving superior convergence to true Generalized Nash Equilibria.
 Furthermore, successful deployment on a fleet of miniature autonomous racing vehicles confirms that the algorithm is robust to infeasible initializations and sufficiently fast for real-time, receding-horizon control in complex, multi-agent environments.
 Future work will explore distributed implementations of the solver and further theoretical characterizations of the modified descent dynamics in asymmetric games.

\balance
\bibliographystyle{ieeetran}
\bibliography{refs}

\end{document}